\documentclass[review,12pt]{elsarticle}

\usepackage{graphicx}
\usepackage[hyphens]{url}

\usepackage{algorithm}
\usepackage{algorithmic}

\usepackage{booktabs}

\usepackage{amsmath}
\usepackage{amssymb}
\usepackage{bbm}

\usepackage[colorlinks=true, linkcolor=blue, citecolor=blue, urlcolor=blue]{hyperref}

\newtheorem{proposition}{Proposition}
\newdefinition{remark}{Remark}
\newproof{pf}{Proof}

\journal{Pattern Recognition}

\begin{document}

\begin{frontmatter}

\title{Quantifying Volumetric Risk: Class-Aware Asymmetric Weighted Conformal Prediction for 3D Medical Image Segmentation}

\author{Shadi Alijani}
\author{Fereshteh Aghaee Meibodi}
\author{Homayoun Najjaran\corref{cor1}}
\ead{najjaran@uvic.ca}
\cortext[cor1]{Corresponding author}
\affiliation{organization={University of Victoria},
            addressline={800 Finnerty Road},
            city={Victoria},
            state={BC},
            country={Canada}}

\begin{abstract}
Reliable volumetric segmentation is critical for clinical diagnostics, yet foundation models such as MedSAM remain deterministic and lack calibrated uncertainty under distribution shift. Existing conformal prediction methods offer statistical guarantees but are frequently applied in 2D and assume symmetric error distributions, so they do not capture the class-specific biases that arise in 3D multi-class segmentation. We propose Class-Aware Asymmetric Weighted Conformal Prediction (CA-WCP), which combines latent-space density-ratio weighting for covariate shift with directional quantiles for the lower and upper volume bounds, and scales each bound by a class-specific asymmetry factor derived from validation-set false-positive and false-negative rates. We prove that CA-WCP retains the weighted-exchangeability marginal coverage guarantee for every class, and we evaluate it on 3D brain tumor segmentation (BraTS 2020) and on a synthetic multi-organ CT benchmark constructed under covariate shift. On both benchmarks the 95\% Clopper--Pearson interval for the observed coverage of CA-WCP contains the nominal 90\% level for every semantic class, while interval width is reduced by 8--14\% relative to symmetric weighted conformal prediction. We further encode the calibrated intervals into structured prompts for a multimodal large language model to produce uncertainty-conditioned radiology reports, linking distribution-shift-aware uncertainty quantification to interpretable clinical communication.
\end{abstract}


\end{frontmatter}

%% Line numbering is supplied by the [review] class option -- do NOT call
%% \linenumbers here (lineno is not loaded by the class and it errors out).

% ---------------------------------------------------------------
\section{Introduction}
\label{sec:intro}

Reliable volumetric segmentation is essential for clinical diagnostics and treatment planning in medical imaging. Foundation models such as the Segment Anything Model (SAM)~\cite{kirillov2023segment} have demonstrated strong zero-shot transfer capabilities, leading to adaptations like MedSAM~\cite{ma2024segment} that apply these models to the medical domain via large-scale pretraining and parameter-efficient fine-tuning. To address volumetric clinical applications, recent extensions further incorporate 3D context, either by redesigning the architecture for native volumetric processing or by introducing lightweight 3D adapters and prompt propagation mechanisms~\cite{wang2025sam, zhu2024medical}.

However, most of these methods remain focused on improving segmentation accuracy, offering little insight into prediction reliability or uncertainty calibration~\cite{begoli2019need, nair2020exploring}, which is a critical requirement for clinical adoption. While prior efforts explored probabilistic or Bayesian models~\cite{kohl2018probabilistic, mehrtash2020confidence, yu2019uncertainty}, conformal prediction (CP)~\cite{vovk2005algorithmic} has emerged as a powerful framework for providing formal statistical guarantees on segmentation reliability.

Despite their success, conventional CP methods face two limitations in this context. First, they rely on the exchangeability assumption, which is often violated in real-world medical settings due to distribution shifts between training and deployment sites. This has motivated recent work in distribution-shift-aware CP that uses adaptive weighting of calibration statistics to restore target coverage~\cite{alijani2025wqlcp, bhatnagar2023improved, kasa2023empirically}. Second, while recent work has explored CP for 3D medical imaging biomarkers~\cite{cheung2026compass, cheung2025metric, wen2024task, wen2025minimax, tassopoulou2026uncertainty}, these methods predominantly target reconstruction metrics or trajectory prediction rather than multi-class volumetric segmentation under covariate shift. Existing volumetric CP formulations largely assume symmetric error distributions, which do not capture the class-specific asymmetric biases observed in multi-class segmentation tasks such as brain tumor analysis, where distinct structures exhibit different false-positive and false-negative patterns.

Separately, large language models (LLMs) have shown strong reasoning in medical question answering and report drafting~\cite{singhal2023large, achiam2023gpt}, but they lack a mechanism for statistical uncertainty quantification.

Motivated by these gaps, our work extends weighted CP to the volumetric 3D multi-class setting and integrates the resulting calibrated intervals with an LLM to produce uncertainty-aware clinical reports. A further practical limitation in the medical CP literature is reliance on single-modality datasets, owing to privacy constraints and annotation costs for multi-center clinical data. To test whether the directional error asymmetry we exploit reflects structure scale rather than brain tumor morphology specifically, we additionally construct a controlled synthetic CT benchmark with anatomically aligned ground-truth masks using a latent rectified flow generator for volumetric medical synthesis~\cite{zhao2026maisi_model, zhao2026maisi_tool}.

Our main contributions are:
\begin{itemize}
    \item A weighted conformal prediction framework for 3D volumetric segmentation that integrates class-conditional asymmetric calibration with MedSAM to provide uncertainty quantification under covariate shift.
    \item Directional quantiles that separately calibrate the lower and upper volume bounds, combined with asymmetry factors derived from validation-set FP/FN rates that add a conservative, class-specific margin on the side of the dominant error.
    \item A proof that CA-WCP retains the marginal coverage guarantee under weighted exchangeability with known likelihood ratios (Proposition~\ref{prop:coverage}, proved in \ref{app:proof}), together with an explicit statement of the approximation incurred when the likelihood ratios are estimated.
    \item An empirical evaluation on BraTS 2020 and on a controlled synthetic multi-organ CT benchmark, reporting exact $k/n$ coverage with Clopper--Pearson intervals for every class.
    \item A framework for uncertainty-conditioned multimodal LLM report generation that encodes calibrated prediction intervals into structured prompts, together with a preliminary assessment by an independent radiologist blinded to the prompt configuration.
\end{itemize}

% ---------------------------------------------------------------
\section{Background}
\label{sec:Background}

\subsection{Foundation Models for 2D and 3D Medical Segmentation}
Medical image analysis has moved from task-specific architectures such as U-Net~\cite{ronneberger2015u} toward foundation models that generalize across datasets and tasks~\cite{khan2025comprehensive, kaliappan2026unified}. Among them, SAM~\cite{kirillov2023segment} enables prompt-driven segmentation with strong zero-shot transfer, motivating specialized adaptations for medical imaging~\cite{alijani2024vision}. MedSAM~\cite{ma2024segment}, SAM-Med2D~\cite{cheng2023sam}, and SAMed~\cite{zhang2023customized} adapt SAM for 2D medical images via large-scale pretraining, systematic fine-tuning, and efficient low-rank adaptation, respectively. To extend these capabilities to 3D, SAM-Med3D~\cite{wang2025sam} and SAM3D~\cite{bui2024sam3d} develop native 3D formulations, while MedSAM-2~\cite{zhu2024medical} uses video-style prompt propagation. Despite their accuracy, these methods largely neglect uncertainty quantification and assume stable data distributions.

\subsection{Uncertainty Quantification in Medical Segmentation}
Quantifying uncertainty is critical for improving the robustness of medical segmentation models~\cite{alijani2024vision, lu2022improving, vazquez2022conformal}. Early approaches included generative multi-hypothesis models such as the Probabilistic U-Net~\cite{kohl2018probabilistic}, semi-supervised consistency regularization~\cite{yu2019uncertainty}, and uncertainty-aware loss functions~\cite{mehrtash2020confidence}. Within the SAM family, uncertainty has been injected through adapters~\cite{jiang2024uncertainty} or specialized fine-tuning, as in I-MedSAM~\cite{wei2024medsam} and U-MedSAM~\cite{wang2024u}. While these methods provide useful heuristics for model confidence, they generally lack the formal statistical guarantees offered by CP.

\subsection{Conformal Prediction and Distribution Shifts}
A complementary line of work uses CP to provide statistical guarantees on segmentation reliability~\cite{zhang2023rr, gade2024impact}. While early CP methods focused on pixel-level uncertainty~\cite{mossina2024conformal}, later extensions introduced morphology-aware prediction sets~\cite{mossina2025conformal}. To mitigate distribution shifts, weighted CP variants have been developed~\cite{tibshirani2019conformal, alijani2025wqlcp}, and TriadNet~\cite{lambert2023triadnet} laid the groundwork for multi-head uncertainty. Conformal methods have also been applied to 3D medical imaging biomarkers, including segmentation-derived metrics, image reconstruction, imaging inverse problems, and longitudinal trajectories~\cite{cheung2026compass, cheung2025metric, wen2024task, wen2025minimax, tassopoulou2026uncertainty}. Many CP approaches for segmentation nonetheless operate in 2D or at region-level volumetry~\cite{lambert2024robust, gade2024impact}, and class effects are typically handled only implicitly through adaptive score functions. Group-conditional constructions instead partition the calibration set explicitly, and recent conformal predictive systems extend this idea to settings in which the error scale differs across groups~\cite{zhang2025mondrian}. Our framework calibrates class effects explicitly and couples them to the directional structure of the volumetric error.

\subsection{LLMs in Medical Imaging and Uncertainty Communication}
LLMs have recently demonstrated strong reasoning capabilities across a range of medical tasks. Med-PaLM~\cite{singhal2023large} and GPT-4~\cite{achiam2023gpt, nori2023capabilities} show near expert-level performance on medical benchmarks, and vision-language models such as LLaVA-Med~\cite{li2023llava} enable image-grounded dialogue. Despite these advances, the integration of statistically calibrated uncertainty quantification into LLM-based medical systems remains limited.

% ---------------------------------------------------------------
\section{Method}
\label{sec:Method}

\textbf{Notation.} $v^*$ denotes ground-truth volume; $v^{\text{low}}, v^{\text{mean}}, v^{\text{up}}$ denote predicted volumes from three heads; $\gamma_{\text{left}}, \gamma_{\text{right}}$ denote asymmetry factors; $q_{\text{left}}, q_{\text{right}}$ denote directional quantiles; $w_i$ denotes density-ratio weights; and $c \in \mathcal{C}$ indexes semantic classes (tumor subregions in brain MRI, abdominal organs in CT). Throughout, $\alpha$ is the conformal miscoverage rate; the Tversky loss parameters are written $\alpha_T, \beta_T$.

\subsection{Problem Formulation}
Given a trained segmentation model $f: \mathcal{X} \rightarrow \mathcal{Y}$ mapping 3D medical images $\mathbf{x} \in \mathcal{X}$ to segmentations $\mathbf{y} \in \mathcal{Y}$, we seek calibrated predictive intervals for structure volumes $\text{vol}(\mathbf{y})$. For multi-class segmentation we require intervals $\Gamma_{\alpha}(\mathbf{x}) = [\ell, u]$ such that the true volume $v^*$ is contained with probability at least $1-\alpha$:
\begin{equation}
\mathbb{P}(v^* \in [\ell, u]) \geq 1 - \alpha
\end{equation}
where $\alpha$ is the desired miscoverage rate ($\alpha = 0.10$ for 90\% coverage throughout).

\subsection{Multi-Head Volumetric Backbone}
\label{sec:backbone}

Following TriadNet~\cite{lambert2023triadnet}, we adopt a multi-head architecture producing three segmentation estimates per input: a lower bound $f_{\text{low}}$, a mean estimate $f_{\text{mean}}$, and an upper bound $f_{\text{up}}$. The heads share an encoder but have separate decoder branches. They are trained with the Tversky loss~\cite{salehi2017tversky} using different $\alpha_T$ and $\beta_T$ to control false positive and false negative penalties:
\begin{equation}\label{eq:tversky}
\mathcal{L}_T(\mathbf{p}, \mathbf{y}) = 1 - \frac{\sum_i p_i y_i}{\sum_i p_i y_i + \alpha_T \sum_i p_i(1-y_i) + \beta_T \sum_i (1-p_i)y_i}
\end{equation}
where the fraction is the Tversky index (a generalization of the Dice coefficient), so that minimizing $\mathcal{L}_T$ maximizes overlap.

The binarization thresholds $\tau_{\text{low}}, \tau_{\text{mean}}, \tau_{\text{up}}$ are fixed before calibration and are applied identically to the calibration and test volumes. CA-WCP is therefore a strictly post-hoc wrapper: it does not modify the segmentation masks, and segmentation accuracy (Dice) is by construction identical across all CP variants built on the same backbone.

\subsection{Weighted Conformal Prediction}
Given a calibration set $\mathcal{D}_{\text{cal}} = \{(\mathbf{x}_i, v_i^*)\}_{i=1}^{n}$ and test point $\mathbf{x}_{\text{test}}$, standard split CP computes conformal scores $s_i = \max(v_i^* - v_i^{\text{up}},\, v_i^{\text{low}} - v_i^*)$ and forms the interval $[v_{\text{test}}^{\text{low}} - q,\; v_{\text{test}}^{\text{up}} + q]$. To handle covariate shift between calibration and test distributions, weighted conformal prediction~\cite{tibshirani2019conformal, alijani2025wqlcp} assigns importance weights $w_i$ to calibration samples. The weights are estimated with a probabilistic classifier $\phi$ trained to distinguish test latent features (label 1) from calibration latent features (label 0). Writing $\hat{p}(\mathbf{z}) = \mathbb{P}_{\phi}(\text{label}=1 \mid \mathbf{z})$, the density ratio at latent feature $\mathbf{z}_i$ is estimated as
\begin{equation}\label{eq:weights}
w_i = \frac{n_{\text{cal}}}{n_{\text{test}}} \cdot \frac{\hat{p}(\mathbf{z}_i)}{1 - \hat{p}(\mathbf{z}_i)},
\end{equation}
where the prior factor $n_{\text{cal}}/n_{\text{test}}$ corrects for unequal set sizes; $w_{\text{test}}$ is obtained in the same way from $\mathbf{z}_{\text{test}}$. Because Eq.~(\ref{eq:weighted_quantile}) uses only normalized weights, any factor common to all samples, including this prior correction, cancels in the quantile.

For a target level $\beta \in (0,1)$, the weighted quantile $q_w(\beta)$ is defined as the infimum over values carrying sufficient weighted probability mass, with a point mass at $+\infty$ for the test point. Writing $W = \sum_{j=1}^{n} w_j + w_{\text{test}}$ for the total mass and $\tilde{w}_i = w_i / W$, $\tilde{w}_{\text{test}} = w_{\text{test}} / W$ for the normalized weights,
\begin{equation}
\label{eq:weighted_quantile}
q_w(\beta) = \inf \Bigl\{ q \in \mathbb{R} \;:\; \sum_{i=1}^{n} \tilde{w}_i \,\mathbbm{1}[s_i \leq q] \;+\; \tilde{w}_{\text{test}} \,\mathbbm{1}[\infty \leq q] \;\geq\; \beta \Bigr\}
\end{equation}
Symmetric weighted CP instantiates Eq.~(\ref{eq:weighted_quantile}) at $\beta = 1-\alpha$ on the two-sided score $s_i$.

\begin{figure}[!t]
\centering
\includegraphics[width=\columnwidth]{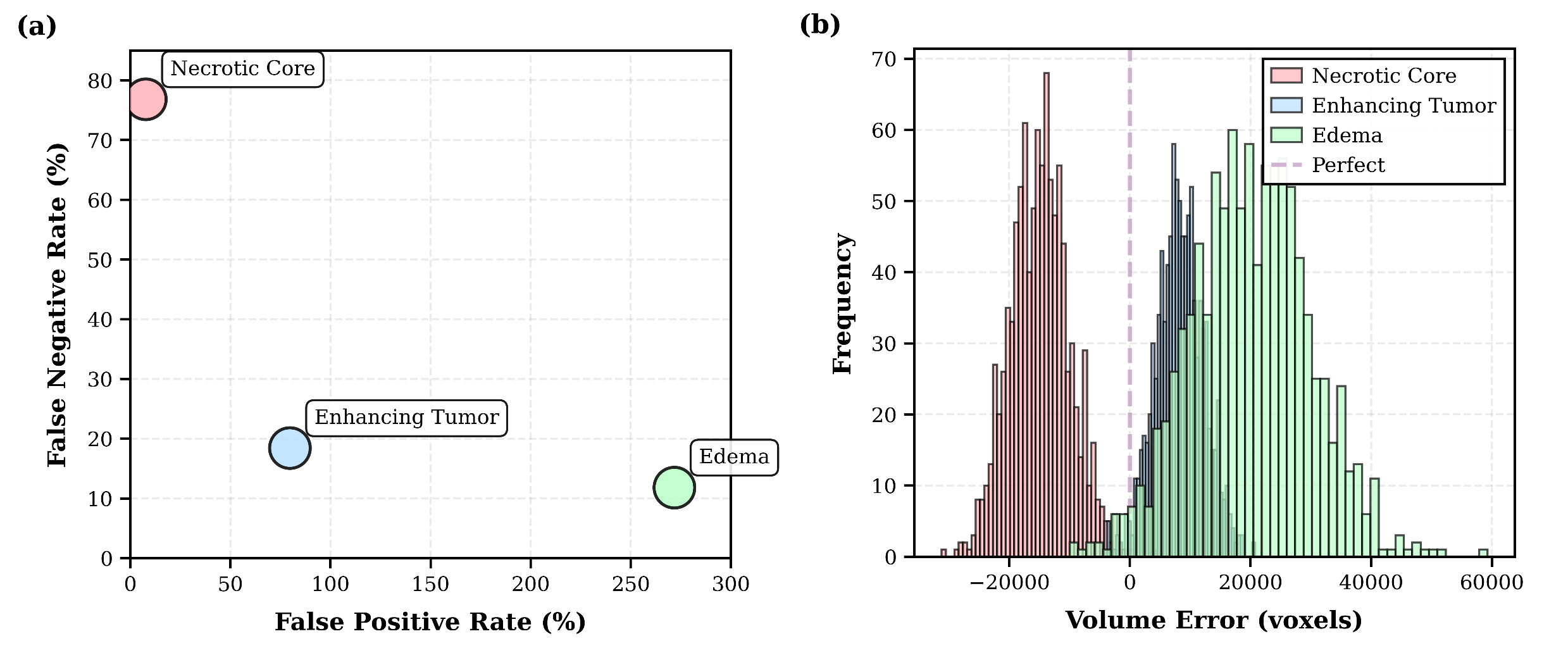}
\caption{Class-specific error patterns revealing asymmetric biases in brain tumor segmentation, measured on the validation split.
(Left) False positive versus false negative rates. Rates are normalized by ground-truth volume $v^*$ (Eq.~(\ref{eq:fpfn_rates})), so values above 100\% are well defined and indicate severe over-segmentation. Necrotic core shows a false-negative bias, edema a large false-positive bias, and enhancing tumor a moderate false-positive bias.
(Right) Distribution of volume errors across classes, showing the skew that motivates asymmetric interval design.}
\label{fig:error_patterns}
\end{figure}

\subsection{Class-Aware Asymmetric Weighted Conformal Prediction}
\label{sec:cawcp_formulation}

\textbf{Motivation.} Split CP assumes symmetric error distributions and produces intervals $[v^{\text{low}} - q,\, v^{\text{up}} + q]$ with equal margins on both sides. Multi-class medical segmentation, however, exhibits class-specific asymmetric biases: different semantic classes show distinct error patterns (Figure~\ref{fig:error_patterns}). Brain tumor segmentation shows systematic under-segmentation of necrotic core and over-segmentation of edema. Such patterns make symmetric intervals inefficient, producing under-coverage in some classes and over-conservatism in others.

\begin{figure}[!t]
\centering
\includegraphics[width=1.0\textwidth]{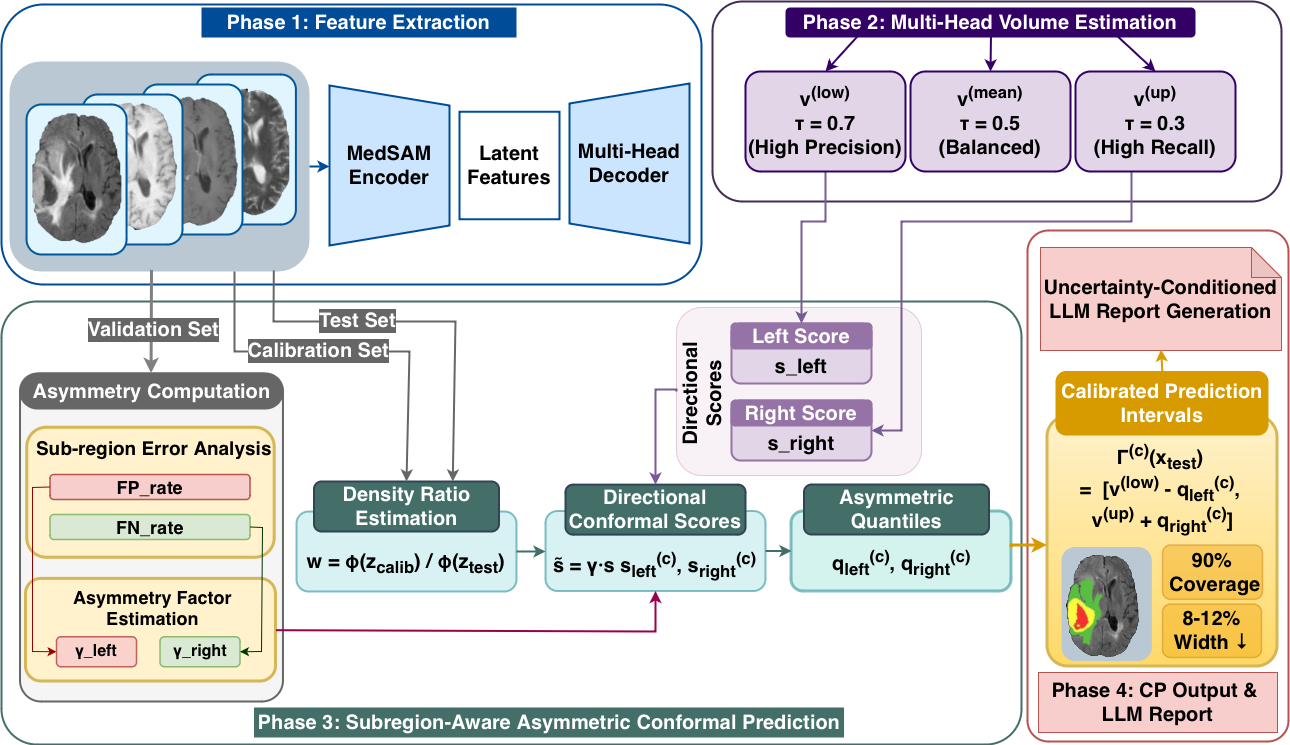}
\caption{Overview of the proposed CA-WCP pipeline: feature extraction, multi-threshold segmentation, and weighted asymmetric conformal prediction.}
\label{fig:workflow}
\end{figure}

Figure~\ref{fig:workflow} gives an overview of the full pipeline.

\textbf{Error-rate definitions.} For class $c$, let $\mathbf{y}^{(c)}$ and $\hat{\mathbf{y}}^{(c)}$ denote the ground-truth and predicted (mean-head) binary masks. We define the volumetric false positive and false negative rates, both normalized by ground-truth volume:
\begin{equation}\label{eq:fpfn_rates}
\text{FP}_{\text{rate}}^{(c)} = \frac{|\hat{\mathbf{y}}^{(c)} \setminus \mathbf{y}^{(c)}|}{|\mathbf{y}^{(c)}|},
\qquad
\text{FN}_{\text{rate}}^{(c)} = \frac{|\mathbf{y}^{(c)} \setminus \hat{\mathbf{y}}^{(c)}|}{|\mathbf{y}^{(c)}|}.
\end{equation}
Both rates are averaged over $\mathcal{D}_{\text{val}}$.

\textbf{Method.} We learn class-specific asymmetric quantiles $(q_{\text{left}}^{(c)}, q_{\text{right}}^{(c)})$ for each $c \in \mathcal{C}$. The predictive interval is
\begin{equation}\label{eq:interval}
\Gamma^{(c)}(\mathbf{x}_{\text{test}}) = \bigl[v_{\text{test}}^{(c,\text{low})} - \gamma_{\text{left}}^{(c)} q_{\text{left}}^{(c)},\; v_{\text{test}}^{(c,\text{up})} + \gamma_{\text{right}}^{(c)} q_{\text{right}}^{(c)}\bigr]
\end{equation}
calibrated from directional conformal scores computed separately:
\begin{align}
s_{\text{left}}^{(c)} &= \max(v^{(c,\text{low})} - v^{*,(c)},\, 0) \label{eq:score_left} \\
s_{\text{right}}^{(c)} &= \max(v^{*,(c)} - v^{(c,\text{up})},\, 0) \label{eq:score_right}
\end{align}

The asymmetry factors are determined on $\mathcal{D}_{\text{val}}$ by a piecewise rule in $\text{FN}_{\text{rate}}^{(c)}$ and $\text{FP}_{\text{rate}}^{(c)}$:
\begin{align}
\gamma_{\text{left}}^{(c)} &= \begin{cases}
1.0 + 2.0 \cdot \text{FN}_{\text{rate}}^{(c)} & \text{if } \text{FN}_{\text{rate}}^{(c)} > 0.5 \\
1.0 + 0.5 \cdot (\text{FP}_{\text{rate}}^{(c)} - 1.0) & \text{else if } \text{FP}_{\text{rate}}^{(c)} \geq 1.5 \\
1.0 + 0.4 \cdot (\text{FP}_{\text{rate}}^{(c)} - 0.7) & \text{else if } \text{FP}_{\text{rate}}^{(c)} \in (0.7, 1.5) \\
1.0 & \text{otherwise}
\end{cases} \label{eq:asymmetry_left} \\
\gamma_{\text{right}}^{(c)} &= \begin{cases}
1.0 + 2.0 \cdot \text{FN}_{\text{rate}}^{(c)} & \text{if } \text{FN}_{\text{rate}}^{(c)} > 0.5 \\
1.0 & \text{otherwise}
\end{cases} \label{eq:asymmetry_right}
\end{align}
Branches are evaluated top to bottom and the first matching condition is applied, so the rule is well defined when several conditions hold simultaneously. Every branch yields $\gamma \geq 1$, a property used in the coverage analysis below. The breakpoints and coefficients are fixed on $\mathcal{D}_{\text{val}}$ prior to calibration and are not tuned on $\mathcal{D}_{\text{cal}}$ or on the test data.

The directional quantiles are computed from the \emph{unscaled} directional scores, each at level $1-\alpha/2$ using the density-ratio weighted quantile of Eq.~(\ref{eq:weighted_quantile}):
\begin{align}
q_{\text{left}}^{(c)} &= q_w\bigl(1-\tfrac{\alpha}{2}\bigr) \text{ computed on } \{s_{\text{left}}^{(c)}(i)\}_{i \in \mathcal{D}_{\text{cal}}} \text{ with weights } \{w_i\} \label{eq:qleft} \\
q_{\text{right}}^{(c)} &= q_w\bigl(1-\tfrac{\alpha}{2}\bigr) \text{ computed on } \{s_{\text{right}}^{(c)}(i)\}_{i \in \mathcal{D}_{\text{cal}}} \text{ with weights } \{w_i\} \label{eq:qright}
\end{align}
The per-direction level $1-\alpha/2$ is what makes the union bound in \ref{app:proof} yield overall coverage $1-\alpha$. Setting $\gamma_{\text{left}}^{(c)} = \gamma_{\text{right}}^{(c)} = 1$ yields the Directional WCP baseline; the comparison between the Directional WCP and CA-WCP rows of Table~\ref{tab:results} therefore constitutes a direct ablation of the asymmetry mechanism, holding the backbone, the weighting, and the calibration split fixed.

\textbf{Relation to prior CP techniques.} CA-WCP composes three established conformal ingredients: density-ratio weighted CP for covariate shift~\cite{tibshirani2019conformal, alijani2025wqlcp}; directional asymmetric scores, in the spirit of conformalized quantile regression~\cite{romano2019conformalized}; and class-conditional (Mondrian) calibration~\cite{vovk2005algorithmic}, as also used for subgroup-equitable coverage in medical imaging~\cite{lu2022fair}. Our contribution is not any single ingredient but their integration in the volumetric multi-class setting, together with the validation-derived asymmetry factors $\gamma^{(c)}$ that couple interval geometry to the observed per-class FP/FN error structure. That coupling is absent from a naive composition of the three techniques, and it is what distinguishes CA-WCP from applying Mondrian calibration to a symmetric weighted score. In the nested-set formulation of conformal prediction~\cite{gupta2022nested}, the construction corresponds to a two-parameter family of nested intervals whose lower and upper radii are calibrated separately, rather than the single-radius family underlying symmetric split CP.

\textbf{Coverage guarantee.}

\begin{proposition}[Marginal Coverage]
\label{prop:coverage}
Fix a class $c$ and let $\gamma_{\mathrm{left}}^{(c)}, \gamma_{\mathrm{right}}^{(c)} \geq 1$ be hyperparameters learned from $\mathcal{D}_{\mathrm{val}}$, independent of $\mathcal{D}_{\mathrm{cal}}$ and of the test point. Suppose the calibration samples $\{(\mathbf{x}_i, v_i^*)\}_{i=1}^n$ and the test point $(\mathbf{x}_{\mathrm{test}}, v_{\mathrm{test}}^*)$ are weighted exchangeable in the sense of~\cite{tibshirani2019conformal} (likelihood ratio $\mathrm{d}\tilde{P}_X / \mathrm{d}P_X$), with weights proportional to the true likelihood ratio between the test and calibration covariate distributions, and that each directional quantile is computed at level $1-\alpha/2$ via Eq.~(\ref{eq:weighted_quantile}). Then, conditional on $\gamma^{(c)}$,
\[
\mathbb{P}\bigl(v_{\mathrm{test}}^{*,(c)} \in \Gamma^{(c)}(\mathbf{x}_{\mathrm{test}}) \,\bigm|\, \gamma^{(c)}\bigr) \geq 1 - \alpha .
\]
\end{proposition}

\begin{remark}
The guarantee is marginal and holds separately for each class $c$; we do not claim simultaneous coverage of all classes at level $1-\alpha$, which would require either a correction across $|\mathcal{C}|$ targets or an explicitly multivariate construction such as copula-based multi-target conformal prediction~\cite{messoudi2021copula}. Proposition~\ref{prop:coverage} assumes the true likelihood-ratio weights. In practice $\phi$ is estimated from latent features, so the weights are approximate and the guarantee degrades by a term governed by the weight-estimation error~\cite{tibshirani2019conformal, barber2021predictive}. The factors $\gamma^{(c)} \geq 1$ do not alter this statement: they can only widen the interval, on the side where validation errors concentrate, and so act as a conservative margin when the weights are estimated, but they are neither derived from nor a bound on the weight-estimation error. We therefore report exact $k/n$ coverage with 95\% Clopper--Pearson intervals throughout Sec.~\ref{sec:Experiments}; on both benchmarks the intervals for CA-WCP contain the nominal 90\% level for every class.
\end{remark}

\subsection{Uncertainty-Conditioned Multimodal LLM Report Generation}
\label{sec:llm_report}

To translate the quantitative uncertainty estimates into clinically readable form, we condition a multimodal LLM on the calibrated volumetric intervals.

\textbf{Confidence Level Derivation.} For each class $c$ we compute the relative interval width (RW), normalizing the interval against the mean-head volume and using the same asymmetry-scaled endpoints as Eq.~(\ref{eq:interval}):
\begin{equation}
    \text{RW}^{(c)} = \frac{\bigl(v^{(c,\text{up})} + \gamma_{\text{right}}^{(c)} q_{\text{right}}^{(c)}\bigr) - \bigl(v^{(c,\text{low})} - \gamma_{\text{left}}^{(c)} q_{\text{left}}^{(c)}\bigr)}{v^{(c,\text{mean})}}
\end{equation}
RW is mapped to discrete confidence levels using thresholds set from quantiles of the validation-set RW distribution: High (RW $< 35\%$), Moderate ($35\%$--$60\%$), and Low ($> 60\%$). All volumes are converted to milliliters using image spacing before being passed to the LLM.

\textbf{Uncertainty-Conditioned Prompting.} We employ GPT-4o at temperature 0.0 to minimize sampling variability, with a multimodal prompt containing (1) the MR image, (2) the generated segmentation mask, and (3) a structured text representation of the conformal outputs including the derived confidence level. The prompt instructs the model to adapt its language to the confidence level, using more cautious phrasing for low-confidence predictions. Reports follow radiology convention with a ``Findings'' section and an ``Impression'' section.

% ---------------------------------------------------------------
\section{Experiments}
\label{sec:Experiments}

\subsection{Implementation Details}
We adapt MedSAM~\cite{ma2024segment} for brain tumor segmentation using multi-modal MR images (T1, T1ce, T2, FLAIR). The image encoder is TinyViT with output dimensions $256 \times 64 \times 64$. We fine-tune the multi-head decoder jointly on the BraTS 2020 training split for 200 epochs using Adam (base learning rate $3\times10^{-5}$, cosine decay, batch size 2). Following TriadNet~\cite{lambert2023triadnet}, the three heads are parameterized by varying the Tversky weights: the lower head emphasizes precision with $(\alpha_T,\beta_T)=(0.3,0.7)$, the mean head is symmetric at $(0.5,0.5)$, and the upper head emphasizes recall at $(0.7,0.3)$. Logits are binarized with fixed thresholds $\tau_{\text{low}}=0.7$, $\tau_{\text{mean}}=0.5$, $\tau_{\text{up}}=0.3$, held constant across calibration and test.

For latent feature extraction we take global-average-pooled TinyViT features after the penultimate block and apply $L_2$ normalization, giving 256-dimensional vectors. The density-ratio estimator $\phi$ is a logistic regression classifier trained on calibration versus test latent features ($n_{\text{cal}}=66$, $n_{\text{test}}=73$, with the unequal set sizes handled by the prior factor in Eq.~(\ref{eq:weights})) and evaluated with 20-fold stratified cross-validation; we report the mean ROC-AUC ($0.78 \pm 0.03$) as a measure of separability. Because $\phi$ is a linear model on normalized features, no kernel bandwidth is involved. We fix the random seed to 42 for data splits, network initialization, and classifier training.

\subsection{Experimental Setup}
\textbf{Dataset.} We use BraTS 2020~\cite{menze2014multimodal, bakas2017advancing, bakas2018identifying}, a multi-class task with three tumor classes (necrotic core, edema, enhancing tumor) exhibiting distinct error patterns. We use a four-way split: training ($n=200$), validation ($n=30$), calibration ($n=66$), and test ($n=73$ from unseen institutions), with no overlap between sets. Calibration and test are drawn from different institutional cohorts, inducing a natural covariate shift that the density-ratio estimator is designed to correct.

\textbf{Synthetic CT benchmark.} To isolate the dependence of the directional asymmetry on structure scale, we construct a synthetic multi-organ CT dataset using NV-Generate-CTMR~\cite{zhao2026maisi_tool}, built on the MAISI-v2 latent rectified flow architecture~\cite{zhao2026maisi_model}. We generate 500 whole-body CT volumes with paired anatomically aligned masks: 300 are used to fine-tune the MedSAM multi-head backbone, and the remaining 200 are partitioned into validation ($n=50$), calibration ($n=50$), and test ($n=100$) sets. To simulate an acquisition shift, the test set is generated under different textual conditioning prompts and structure-scale parameters. The asymmetry factors $\gamma$ are computed from scratch on the CT validation set rather than reused from BraTS. We evaluate three organs spanning a range of scales: liver (large), kidneys (medium), and spleen (small).

\textbf{Metrics.} We report (1) \textit{Coverage}, the fraction of test cases with $v^* \in \Gamma(\mathbf{x})$, given as an exact $k/n$ fraction with a 95\% Clopper--Pearson interval; (2) \textit{Interval Width}, the mean $u - \ell$ in mL; and (3) \textit{Dice}, reported once per backbone since CP is post-hoc and does not alter the masks.

\textbf{Baselines.} We compare against (1) \textit{Unweighted CP}, a single shared quantile with no density-ratio correction; (2) \textit{Symmetric WCP}~\cite{lambert2024robust}, density-ratio weighting with symmetric intervals computed on the maximum directional score; and (3) \textit{Directional WCP} ($\gamma \equiv 1$), which uses the directional quantiles of Eqs.~(\ref{eq:qleft})--(\ref{eq:qright}) but applies no asymmetry factors. All three share the backbone, the density-ratio weights, and the calibration split with CA-WCP.

\subsection{Results}

Table~\ref{tab:results} reports per-class coverage and interval width on BraTS 2020. With $n=73$ test subjects, coverage is quantized to multiples of $1/73$; the achievable values bracketing the 90\% target are $65/73 = 89.0\%$ and $66/73 = 90.4\%$.

Symmetric WCP falls below the target on edema and over-covers enhancing tumor, where its Clopper--Pearson interval [95.1--100.0] excludes the nominal level, consistent with the class-specific error skew in Figure~\ref{fig:error_patterns}. Replacing the single max-score quantile by two directional quantiles (Directional WCP, $\gamma \equiv 1$) removes most of this inefficiency: mean width falls from 9.7 to 6.7\,mL (Table~\ref{tab:main_results}), because the larger error direction no longer sets both margins. Its point coverage, however, lies below the target for all three classes (87.7\%, 84.9\%, and 89.0\%), although each Clopper--Pearson interval still contains 90\% at $n=73$.

CA-WCP applies the asymmetry factors computed on $\mathcal{D}_{\text{val}}$ by Eqs.~(\ref{eq:asymmetry_left})--(\ref{eq:asymmetry_right}): $\gamma_{\text{left}}=\gamma_{\text{right}}=2.54$ for necrotic core (FN-driven branch), $\gamma_{\text{left}}=1.86$ and $\gamma_{\text{right}}=1.00$ for edema (FP-driven branch), and $\gamma_{\text{left}}=1.08$ and $\gamma_{\text{right}}=1.00$ for enhancing tumor.
These factors raise point coverage to $66/73 = 90.4\%$ for every class at a mean width of 8.5\,mL, which is 12\% narrower than Symmetric WCP and 27\% wider than Directional WCP. The asymmetry factors therefore trade part of the efficiency gained by the directional quantiles for coverage at or above the nominal level; with a single calibration/test split of this size, the coverage difference between Directional WCP and CA-WCP is not statistically resolved (Sec.~\ref{sec:limitations}).

\begin{table}[htbp]
\centering
\caption{Per-class results on BraTS 2020 (target coverage 90\%, $n=73$). Coverage is an exact $k/73$ fraction with a 95\% Clopper--Pearson interval. Width in mL.}
\label{tab:results}
\scriptsize
\setlength{\tabcolsep}{2pt}
\begin{tabular}{lcc|cc|cc}
\toprule
& \multicolumn{2}{c|}{\textbf{Necrosis}}
& \multicolumn{2}{c|}{\textbf{Edema}}
& \multicolumn{2}{c}{\textbf{Enhancing}} \\
\cmidrule(lr){2-3} \cmidrule(lr){4-5} \cmidrule(lr){6-7}
\textbf{Method}
& Cov (\%) [95\% CI] & Wid.
& Cov (\%) [95\% CI] & Wid.
& Cov (\%) [95\% CI] & Wid. \\
\midrule
Unweighted CP & 84.9 [74.6--92.2] & 4.2 & 78.1 [66.9--86.9] & 21.5 & 97.3 [90.5--99.7] & 6.8 \\
Symmetric WCP & 90.4 [81.2--96.1] & 3.9 & 86.3 [76.2--93.2] & 19.2 & 100.0 [95.1--100.0] & 5.9 \\
Directional WCP ($\gamma\equiv1$) & 87.7 [77.9--94.2] & 2.8 & 84.9 [74.6--92.2] & 12.5 & 89.0 [79.5--95.1] & 4.8 \\
\textbf{Ours (CA-WCP)} & 90.4 [81.2--96.1] & 3.6 & 90.4 [81.2--96.1] & 16.8 & 90.4 [81.2--96.1] & 5.1 \\
\midrule
\multicolumn{7}{@{}l}{\textit{Backbone Dice (identical for all rows): necrosis 0.63, edema 0.81, enhancing 0.88}} \\
\bottomrule
\end{tabular}
\end{table}

\begin{table}[htbp]
\centering
\caption{Aggregate results on BraTS 2020 with a fixed MedSAM backbone. Coverage and width are the unweighted mean of the three per-class values in Table~\ref{tab:results}. Dir.~$q$: directional quantiles; Class~$\gamma$: asymmetry factors; Shift: density-ratio weighting. Bold coverage denotes the value closest to the nominal 90\% target; coverage above the target reflects over-conservatism, and a smaller width obtained at coverage below the target is not a like-for-like efficiency gain.}
\label{tab:main_results}
\scriptsize
\setlength{\tabcolsep}{3pt}
\renewcommand{\arraystretch}{1.05}
\begin{tabular}{@{}lcccccc@{}}
\toprule
\textbf{Method} & \textbf{Dice} & \textbf{Cov. (\%)} & \textbf{Wid. (mL)} & \textbf{Dir. $q$} & \textbf{Class $\gamma$} & \textbf{Shift} \\
\midrule
MedSAM~\cite{ma2024segment} (no CP) & 0.77 & -- & -- & -- & -- & -- \\
\midrule
+ Unweighted CP & 0.77 & 86.8 & 10.8 & No & No & No \\
+ Symmetric WCP~\cite{lambert2024robust} & 0.77 & 92.2 & 9.7 & No & No & Yes \\
+ Directional WCP ($\gamma\equiv1$) & 0.77 & 87.2 & 6.7 & Yes & No & Yes \\
\textbf{+ Ours (CA-WCP)} & 0.77 & \textbf{90.4} & 8.5 & Yes & Yes & Yes \\
\bottomrule
\end{tabular}
\end{table}

\textbf{Synthetic CT benchmark.} Table~\ref{tab:synthetic_ct_results} reports the same comparison on the synthetic multi-organ CT data under the induced conditioning shift. The symmetric weighted baseline shows the scale-dependent pattern seen in MRI: over-coverage of the largest structure and under-coverage of the smallest. CA-WCP moves all three organs toward the nominal level while reducing mean width. Because the shift mechanism and the anatomical scale range are fixed by construction here, this result separates the dependence of the asymmetry on structure scale from many of the confounds present in acquired multi-center data.

\begin{table}[htbp]
\centering
\caption{Synthetic multi-organ CT benchmark under induced conditioning shift (target coverage 90\%, $n=100$). Coverage is an exact $k/100$ fraction with a 95\% Clopper--Pearson interval. Width in mL.}
\label{tab:synthetic_ct_results}
\scriptsize
\setlength{\tabcolsep}{3pt}
\renewcommand{\arraystretch}{1.05}
\begin{tabular}{@{}l cc | cc@{}}
\toprule
& \multicolumn{2}{c|}{\textbf{Symmetric WCP}} & \multicolumn{2}{c@{}}{\textbf{Ours (CA-WCP)}} \\
\cmidrule(lr){2-3} \cmidrule(lr){4-5}
\textbf{Organ (Scale)} & \textbf{Cov (\%) [95\% CI]} & \textbf{Wid.} & \textbf{Cov (\%) [95\% CI]} & \textbf{Wid.} \\
\midrule
Liver (large)   & 98.0 [93.0--99.8] & 145.2 & \textbf{91.0 [83.6--95.8]} & \textbf{128.4} \\
Kidneys (medium)& 89.0 [81.2--94.4] & 32.1  & \textbf{90.0 [82.4--95.1]} & \textbf{29.5}  \\
Spleen (small)  & 84.0 [75.3--90.6] & 18.5  & \textbf{90.0 [82.4--95.1]} & \textbf{16.8}  \\
\midrule
\multicolumn{5}{@{}l}{\textit{Backbone Dice: liver 0.92, kidneys 0.88, spleen 0.85}} \\
\bottomrule
\end{tabular}
\end{table}

\subsection{Uncertainty-Conditioned Report Generation}
\label{sec:llm_eval}

We compare two prompt configurations that differ in the information supplied to the model: Image+Segmentation, and Ours (Image+Segmentation+Uncertainty), the latter including the calibrated intervals and the derived confidence level. Generated reports were compared against reference reports for 10 test cases. Because BraTS lacks native radiology reports, the reference reports were written by a board-certified radiologist who mapped the ground-truth segmentation masks to a standardized clinical template. A second, independent radiologist, blinded to the prompt configuration, rated the outputs on factual accuracy and on concordance between the hedging language and the underlying interval width (Table~\ref{tab:llm_evaluation}). The BLEU-4 score of 0.42 is high for clinical text; it follows from prompting the LLM to produce structured findings from quantitative intervals and scoring them against templated references, which constrains lexical variance compared with open-ended dictation. Supplying the calibrated intervals raises uncertainty concordance from 2.1 to 4.6, indicating that the model's expression of confidence tracks the conformal bounds rather than the visual appearance of the segmentation alone. Figure~\ref{fig:3d_visualization} shows a representative case with its generated report.

\begin{table}[htbp]
\centering
\caption{Assessment of uncertainty-conditioned report generation on 10 BraTS test cases. BLEU-4 is computed against the templated reference reports; factual accuracy and uncertainty concordance are rated on a 1--5 scale by an independent radiologist blinded to the prompt configuration.}
\label{tab:llm_evaluation}
\scriptsize
\setlength{\tabcolsep}{4pt}
\renewcommand{\arraystretch}{1.1}
\begin{tabular}{@{}l ccc@{}}
\toprule
\textbf{Prompt configuration} & \textbf{BLEU-4} & \textbf{Factual Acc.} & \textbf{Uncert. Concord.} \\
\midrule
Image+Seg  & 0.28 & 3.8 / 5.0 & 2.1 / 5.0 \\
\textbf{Ours (Image+Seg+UQ)} & \textbf{0.42} & \textbf{4.8 / 5.0} & \textbf{4.6 / 5.0} \\
\bottomrule
\end{tabular}
\end{table}

\begin{figure}[htbp]
\centering
\includegraphics[width=\columnwidth]{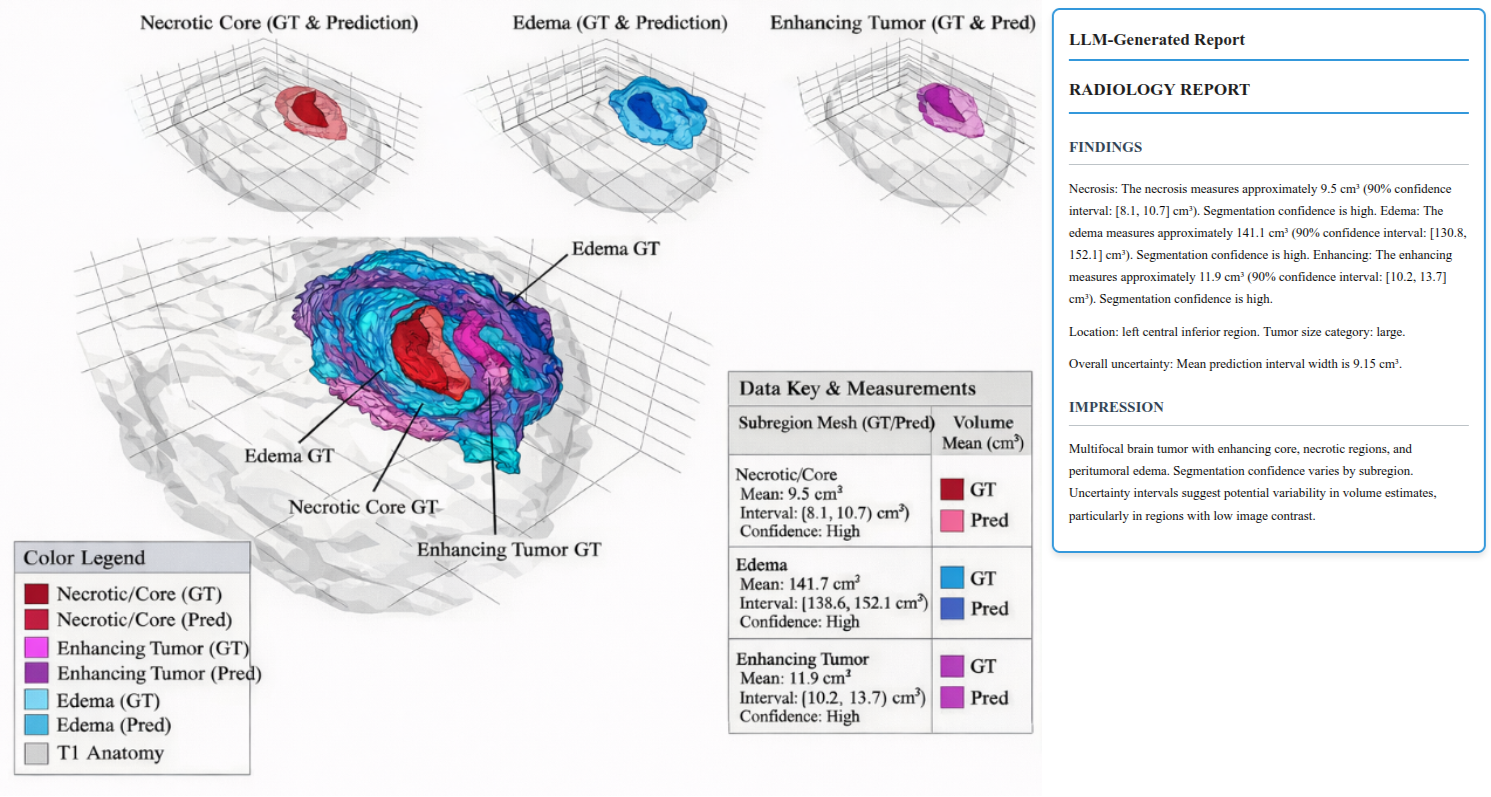}
\caption{Representative patient: ground-truth and predicted tumor subregions overlaid on the T1-weighted MR anatomy, with the uncertainty-conditioned report generated from the calibrated volumetric intervals.}
\label{fig:3d_visualization}
\end{figure}

% ---------------------------------------------------------------
\section{Discussion}
\label{sec:discussion}

\textbf{Analysis of asymmetric calibration.} Comparing Symmetric WCP, Directional WCP ($\gamma \equiv 1$), and CA-WCP in Table~\ref{tab:results} separates the two mechanisms. The directional quantiles supply the efficiency: they prevent the dominant error direction from widening both margins, which reduces mean width by roughly 30\% relative to Symmetric WCP. The asymmetry factors then add a conservative, class-specific margin on the side where validation errors concentrate. Because $\gamma \geq 1$, this margin can only widen the interval and cannot invalidate the guarantee of Proposition~\ref{prop:coverage}; in practice it compensates for the approximate density-ratio weights, under which the $\gamma \equiv 1$ intervals fell slightly below the target in point coverage. Whether that shortfall is systematic or a finite-sample effect cannot be decided from a single split of 73 test subjects, so we report the width cost of the margin explicitly and leave the choice between the two operating points to the user. On the synthetic CT benchmark, CA-WCP shows the same qualitative behavior relative to Symmetric WCP: the over-covered large structure loses excess width and the under-covered small structure gains coverage, which suggests that the effect depends on the directional error structure rather than on brain tumor morphology specifically.

Inference cost is identical to standard split CP because the quantiles are precomputed during calibration.

\subsection{Limitations and Future Directions}
\label{sec:limitations}

Our evaluation uses one clinical cohort and one synthetic benchmark. The synthetic CT data are generated by a learned model and therefore inherit its anatomical priors, so while they isolate the effect of structure scale under a controlled shift, validation on acquired multi-center abdominal CT remains necessary before drawing conclusions about deployment on real CT. Test cohorts of $n=73$ and $n=100$ also place a floor on the precision with which coverage can be estimated, which is why we report exact Clopper--Pearson intervals rather than point estimates alone. All results come from a single calibration/test split; in particular, the coverage gap between Directional WCP and CA-WCP lies within the sampling uncertainty at $n=73$, and repeated random splits would be needed to establish whether weight-estimation error causes systematic under-coverage of the $\gamma \equiv 1$ intervals. The reader assessment of generated reports covers 10 cases and a single rater, and is best read as an indication of how uncertainty conditioning changes reporting language rather than as a clinical validation.

Several directions follow naturally: a repeated-split analysis of the Directional WCP and CA-WCP operating points; a sweep of the generator's conditioning gap to map how coverage and width respond to shift magnitude; instance-adaptive asymmetry factors for heterogeneous tumors; and multi-reader evaluation of uncertainty-conditioned reporting on large paired image-report corpora such as MR-RATE~\cite{mrrate2026dataset}.

% ---------------------------------------------------------------
\section{Conclusion}
\label{sec:conclusion}

We introduced CA-WCP, a class-aware asymmetric weighted conformal prediction framework for volumetric uncertainty quantification in 3D medical image segmentation. We showed that combining density-ratio weighting with directional quantiles calibrated at level $1-\alpha/2$ preserves the marginal coverage guarantee under weighted exchangeability, and we gave the proof together with an explicit account of the approximation introduced by estimated weights. Empirically, on BraTS 2020 and on a controlled synthetic multi-organ CT benchmark, the 95\% Clopper--Pearson intervals for observed coverage contain the nominal 90\% level for every class, and average interval widths are 8--14\% smaller than those of the symmetric weighted baseline; most of this gain comes from the directional quantiles, with the asymmetry factors adding a conservative class-specific margin. We further showed that the calibrated intervals can be encoded into prompts for a multimodal LLM to produce uncertainty-conditioned reports whose hedging language tracks the underlying interval width. Because CA-WCP is a post-hoc wrapper, it adds this calibrated, shift-aware uncertainty without modifying the underlying segmentation model.

\appendix

\section{Proof of Proposition~\ref{prop:coverage}}
\label{app:proof}

\begin{pf}
Fix the class $c$ and condition throughout on $\gamma_{\text{left}}^{(c)}, \gamma_{\text{right}}^{(c)}$. Because these are functions of $\mathcal{D}_{\text{val}}$ only, and $\mathcal{D}_{\text{val}}$ is disjoint from $\mathcal{D}_{\text{cal}}$ and from the test point, conditioning on them does not disturb the weighted exchangeability of the calibration scores and the test score.

\textbf{Step 1 (per-direction coverage).} Consider the left direction. The scores $s_{\text{left}}^{(c)}(1),\ldots,s_{\text{left}}^{(c)}(n)$ and the test score $s_{\text{left}}^{(c)}(\text{test})$ are deterministic functions of the corresponding data points and are therefore weighted exchangeable whenever the data are. By the weighted conformal prediction guarantee of Tibshirani et al.~\cite{tibshirani2019conformal}, applied with the true likelihood-ratio weights and the weighted quantile of Eq.~(\ref{eq:weighted_quantile}) evaluated at level $1-\alpha/2$ as in Eq.~(\ref{eq:qleft}),
\[
\mathbb{P}\bigl(s_{\text{left}}^{(c)}(\text{test}) \leq q_{\text{left}}^{(c)}\bigr) \;\geq\; 1 - \tfrac{\alpha}{2}.
\]
The identical argument applied to Eq.~(\ref{eq:qright}) gives $\mathbb{P}\bigl(s_{\text{right}}^{(c)}(\text{test}) \leq q_{\text{right}}^{(c)}\bigr) \geq 1 - \tfrac{\alpha}{2}$.

\textbf{Step 2 (monotone widening).} Every branch of Eqs.~(\ref{eq:asymmetry_left})--(\ref{eq:asymmetry_right}) satisfies $\gamma \geq 1$: the FN-driven branch adds a non-negative multiple of $\text{FN}_{\text{rate}}^{(c)} > 0.5$; the two FP-driven branches add non-negative increments on their respective domains ($\text{FP}_{\text{rate}}^{(c)} \geq 1.5$ and $\text{FP}_{\text{rate}}^{(c)} \in (0.7,1.5)$); and the default branch equals 1. Since the directional scores are non-negative and $q_{\text{left}}^{(c)}, q_{\text{right}}^{(c)} \geq 0$, we have $\gamma_{\text{left}}^{(c)} q_{\text{left}}^{(c)} \geq q_{\text{left}}^{(c)}$ and $\gamma_{\text{right}}^{(c)} q_{\text{right}}^{(c)} \geq q_{\text{right}}^{(c)}$. Hence
\[
\bigl\{ s_{\text{left}}^{(c)}(\text{test}) \leq q_{\text{left}}^{(c)} \bigr\} \subseteq \bigl\{ s_{\text{left}}^{(c)}(\text{test}) \leq \gamma_{\text{left}}^{(c)} q_{\text{left}}^{(c)} \bigr\},
\]
and likewise on the right, so the bounds of Step 1 are preserved after scaling.

\textbf{Step 3 (union bound).} By Eqs.~(\ref{eq:score_left})--(\ref{eq:score_right}) and~(\ref{eq:interval}), the event $v_{\text{test}}^{*,(c)} \notin \Gamma^{(c)}(\mathbf{x}_{\text{test}})$ holds if and only if $s_{\text{left}}^{(c)}(\text{test}) > \gamma_{\text{left}}^{(c)} q_{\text{left}}^{(c)}$ or $s_{\text{right}}^{(c)}(\text{test}) > \gamma_{\text{right}}^{(c)} q_{\text{right}}^{(c)}$, corresponding to the true volume falling below the scaled lower endpoint or above the scaled upper endpoint. Therefore
\begin{align*}
\mathbb{P}\bigl(v_{\text{test}}^{*,(c)} \notin \Gamma^{(c)}(\mathbf{x}_{\text{test}})\bigr)
&\leq \mathbb{P}\bigl(s_{\text{left}}^{(c)}(\text{test}) > \gamma_{\text{left}}^{(c)} q_{\text{left}}^{(c)}\bigr)
   + \mathbb{P}\bigl(s_{\text{right}}^{(c)}(\text{test}) > \gamma_{\text{right}}^{(c)} q_{\text{right}}^{(c)}\bigr) \\
&\leq \tfrac{\alpha}{2} + \tfrac{\alpha}{2} \;=\; \alpha,
\end{align*}
which gives the claim.
\end{pf}

% ---------------------------------------------------------------
\bibliographystyle{elsarticle-num}
\bibliography{PR_bib}

\end{document}